\documentclass{article}
\pdfoutput=1

\usepackage{kr}

\usepackage{times}
\usepackage{soul}
\usepackage{url}
\usepackage[hidelinks]{hyperref}
\usepackage[utf8]{inputenc}
\usepackage[small]{caption}
\usepackage{graphicx}
\usepackage{amsmath}
\usepackage{amsthm}
\usepackage{booktabs}
\usepackage{algorithm}
\usepackage{algorithmic}
\usepackage{enumerate}
\usepackage{amsmath,amsthm,xspace,amssymb}
\usepackage{latexsym}  
\usepackage{wrapfig}
\usepackage{chngcntr}
\usepackage{ifthen}
\usepackage{hyperref}
\usepackage{multirow}
\usepackage[marginpar]{todo}

\newcommand{\commentout}[1]{}
\newcommand{\wbox}{\mbox{$\sqcap$\llap{$\sqcup$}}}

\usepackage{soul}
\usepackage[d]{esvect}

\usepackage{tikz}

\usetikzlibrary{shapes,arrows.meta,calendar,matrix,calc,decorations,snakes}
\tikzset{wave/.style={decorate, decoration={snake, segment length=2mm, amplitude=0.3mm}}}

\hypersetup{linkcolor=blue!70!red,citecolor=black!70!green,colorlinks=true}

\theoremstyle{plain}
\newtheorem{observation}{Remark}[section]
\newtheorem{theorem}[observation]{Theorem}

\theoremstyle{definition}
\newtheorem{definition}[observation]{Definition}

\newtheorem{example}[observation]{Example}
\theoremstyle{remark}
\newtheorem{remark}[observation]{Remark}

\theoremstyle{plain}
\newtheorem{proposition}[observation]{Proposition}

\tikzset{>=Stealth}

\renewcommand\phi{\varphi}
\def\vec#1{\vv{#1}}

\def\cause#1#2{(#1) \rightsquigarrow #2}
\def\causesimp#1#2{#1 \rightsquigarrow #2}

\title{Security Properties as Nested Causal Statements}

\author{%
Matvey Soloviev\and
Joseph Y. Halpern\\
\affiliations
Cornell University\\
\emails
\{msoloviev, halpern\}@cs.cornell.edu
}

\begin{document}
\maketitle

\begin{abstract}
Thinking in terms of causality helps us structure how different parts
of a system depend on each other, and how interventions on one part of
a system may result in changes to other parts. Therefore, formal
models of causality are an attractive tool for reasoning about
security, which concerns itself with safeguarding properties of a
system against interventions that may be malicious. As we show, many
security properties are naturally expressed as nested causal
statements: not only do we consider what caused a particular undesirable
effect, but we also consider what caused this causal relationship
itself to hold. We present a natural way to extend the Halpern-Pearl
(HP) framework for causality to capture such nested causal
statements. This extension adds expressivity, enabling the HP
framework to distinguish between causal scenarios that it could not
previously naturally tell apart. We moreover revisit some design
decisions of the HP framework that were made with non-nested causal
statements in mind, such as the choice to treat specific values of
causal variables as opposed to the variables themselves as causes, and
may no longer be appropriate for nested ones. 
\end{abstract}

\section{Introduction}

Causality is a common feature of our discourse; indeed, it could be
argued that the notion that some circumstance is the cause of another
is fundamental to the way we make sense of the world around us,
providing both explanations of why things are the way they are
and guidance on how we
should act in order to influence their course.  
The standard approach to causality involves
\emph{counterfactuals}: had the cause
not occurred or occurred in a different way than it actually did,
the effect would not have come to pass. We are typically interested in
the effects of \emph{interventions}, which can be viewed as ways of
making the cause occur in a different way.

A system, such as a computer program, mechanism, or collection of
real-life entities such as the staff of a commercial business, is
considered to be \emph{secure} against a potentially 
malicious actor (``attacker'') if no action the attacker can take
could cause some desirable property of the system to be violated. By
viewing the actions available to the attacker as interventions,
this becomes a causal statement: no desirable
property can be violated \emph{because} of something the attacker
did or didn't do. 
We would like to use a formal account of causality to represent and
analyse security properties. This is an attractive approach because,
while many formal models of security have been proposed, especially
in the programming language research community
\cite{GM82,HY87,zdancewic01,SabelfeldMyers03}, these models are typically designed for the
purpose of analysing computer programs only, and are therefore tightly
coupled to a machine model such as state machines or sequences of
memory snapshots. 
In contrast, general-purpose formalisms for causality can 
capture and represent a wider range of scenarios,
including everyday events and their relations. This means
that a causal characterization of a security property can be
evaluated with respect to a real-world scenario, rather than only
the operation of some computer system. Since real-world scenarios are typically
closer to human intuition than computer programs, we expect
this to be helpful in understanding what a particular security
property ``really means''. 

A variety of formal models have been proposed for reasoning about
causal statements and formally defining what it means to be a cause. 
For definiteness, we use the Halpern-Pearl (HP) framework
\cite{Hal47,Hal48}. In this framework, we first represent the
relevant features of the world (``a switch, which can be either
on or off'', ``a lamp, which can be on or off'') as
variables, and the rules that govern their interdependency
(``the lamp is on exactly when the switch is in the on'')
as \emph{structural equations} to form a \emph{causal
  model}, which enables counterfactual reasoning: that is, a
comparison between the real world and a hypothetical alternative which
differs from the actual world in some relevant aspect. 

In the HP framework, potential causes are taken to be conjunctions of
atomic propositions about the values taken by variables, while effects
are taken to be Boolean combinations of atomic propositions.
However, in real-world discourse, we often encounter seemingly
more complex statements, including, in particular, ones where the
purported effect is itself a causal statement (``because I have paid
my electricity bill on time, flipping the light switch on causes the
lamp to turn on''). 
Nested causal statements of this form are particularly common when
discussing notions of authorisation, delegation, and endorsement.

For example, consider a government employee who is authorised to
publicly confirm a classified piece of information, say, that the
government has made contact with an alien civilization.
This employee is corrupt, and may release the information if he
is bribed. The release
of the information, which itself can be interpreted causally
as saying that the fact that contact was made is a cause
of the newspaper article saying that it is,
 is allowed; the employee is
authorised to release information. The bribe on its own is also not
necessarily bad; nothing prohibits acts of kindness towards strangers.
The problem here is that the gift of money caused the alien 
contact to be a cause of the newspaper article.
This corresponds to the notion of \emph{robust declassification}
from the security literature \cite{zdancewic01,MSZ04,nmifc}, which can
be interpreted as more generally saying that whenever secret
causes have public effects, this must not itself 
be due to untrustworthy causes.
 The underlying notion of robustness can
be taken more generally to denote that some security-relevant
circumstance, which could be primitive or itself involve causality, is
not caused by an untrusted party. 

Conversely, nested causes may also render normally unacceptable
causal relationships acceptable. For example, if $A$
performs an action that infringes on $B$'s possessions,
such as redecorating their office, shredding papers, or
changing the settings of $B$'s computer, and $B$ finds this objectionable,
a common defense that $A$ might invoke is to say that they would have stopped
if $B$ had told them to (and $B$ had the opportunity to do so).
In other words, whatever causal relationship there was between
$A$'s intentions and $B$'s property held only because $B$
didn't voice an objection, and thus implicitly endorsed the act.
We can view this line of reasoning, where
an unauthorised cause has a certain effect
because of an authority's implicit or explicit go-ahead, 
as a form of authorisation. The construction
can easily be nested, obtaining examples
where one authority holds a veto over another authority's ability
to hold a veto over a causal relationship, and so on. There
is extensive literature on reasoning about authorisation chains
of this kind using \emph{authorisation logics} \cite{abadiacl},
and nested causality can be used as a way to 
interpret them.

The HP definition of causality does not deal with
the nested causality statements of the type that occur in the examples above.
Fortunately, the HP definition can be extended without change to apply
to nested statements, and seems to give sensible results.
Using nested causality allows us to distinguish causal
scenarios that appear different in a security-relevant way that
cannot be distinguished without nesting unless we 
to unnatural edge cases of the definitions (Theorem \ref{thm:expressive}).
Moreover, having effects that themselves involve counterfactuals
introduces new considerations that are, in a precise sense,
irrelevant to the ``simple'' causal statements the HP model was
designed for.
These considerations suggest a further modification
to the HP definition.
Specifically, the HP definition assumes that both cause and
    effect are formed from Boolean combinations of atomic propositions
    of the form $X=x$: variable $X$ has value $x$.  Thus, it does not say
 ``the switch's position is the cause of whether the
  lamp was on'', 
  but rather ``the switch being on is the cause of the lamp
  being on''.
  The latter statement is more idiomatic in natural language,
but
provides no more information than the former: the HP definition
of actual causality implies that the statement
    ``$X=x$ is a cause of $Y=y$'' can be
  true 
  for only one particular value of $x$ and $y$, namely 
  the values of $X$ and $Y$ in the actual context.

  The picture becomes more interesting with nested causality. For
example, imagine an American vegetable grower who relocated to
Texas. As it happens, the year was marred by climatic irregularities;
while the southern states, including Texas, experienced a drought,
 all remaining states, including the grower's state of origin,
were subject to catastrophic floods instead. The grower now makes the
following statement: ``Because we moved to Texas, the weather caused
our crop to fail.'' This is 
arguably false: had the grower not moved to Texas,
the flooding would have led to crop failure all the same. How would we
make sense of this in the HP framework? What value of the variable
weather are we referring to here?  Naively extending upon the previous
observation, the first thought would be to use the value of weather 
in the actual context, so it becomes ``Because we moved to Texas, the
weather being dry caused our crop to fail.'' But this statement
is true: had the grower stayed in New York, the weather would not have
been dry, and so dry weather could not have caused crop
failure. Plugging in another constant value does not work either;
while ``Because we moved to Texas, the weather being
  very wet caused our crop to fail'' is false, 
due to the statement ``the weather being very wet caused our crop to
fail'' itself being false.  This statement can't be the intended
meaning of ``the weather caused our crop to fail'', because we would
normally take the latter statement to be true! 

We claim that the most reasonable interpretation of this type of statement
is instead as a form of causality that is independent of the concrete
value that the weather takes. This can be interpreted in terms of the
HP notion in terms of existential quantification: ``there exists a
state $v$ of the weather such that the weather being $v$ caused our
crop to fail''.
We return to this point in Section~\ref{sec:variableascauses}.

\section{Review of the HP framework}

We first review the Halpern-Pearl notion of causality. 
The first step is to define 
causal models.

\def\s#1{\mathcal{#1}}

\begin{definition} \label{defn:causalmodel}
  A \emph{causal model} is a pair $(\s{S},\s{F})$,
  consisting 
  of a \emph{signature}
$\s{S}$ and a collection of \emph{structural equations} $\s{F}$ for
this signature. 
The signature $\s{S}$ is a triple $(\s{U},\s{V},\s{R})$; $\s{U}$ is a
nonempty finite
set of \emph{exogenous variables}, to be thought of as external inputs
to the model, or features of the world whose values are determined
outside the model; $\s{V}$ is a nonempty finite set of \emph{endogenous
  variables}, whose 
causal dependencies on each other and on the inputs we wish to
analyse; each variable $W\in \s{U}\cup \s{V}$ can take values from a
finite
range $\s{R}(W)$; $\s{F}$ associates with each endogenous variable $V\in
\s{V}$ a function denoted $F_V$ that determines the value of $V$
in terms of the values of all the other variables in $\s{U} \cup \s{V}$;
thus,  $F_V : \prod_{W\in (\s{U}\cup\s{V} -V)} R(W) \rightarrow
\s{R}(V)$.  We typically write, say, $V = U+X$ rather than $F_V(u,x) =
u+x$ for all $u \in \s{R}(U)$ and $x \in \s{R}(X)$.
\hfill \wbox
\end{definition}

A variable $V$ \emph{depends on} $W$ if the structural
equation for $V$ nontrivially depends on the value taken by $W$:
that is, there are some settings $\vec{z}$ and $\vec{z}'$ of the
variables in $\s{U}$ and $\s{V}$ other than $V$ that 
only differ in the entry corresponding to $W$ such that
$F_V(\vec{u})\neq F_V(\vec{u}')$. 
Note that this notion represents only \emph{immediate}
dependency, and is not transitive: $V$ depending on $W$ and $W$ depending
on $U$ does not imply that $V$ depends on $U$.
 In this paper, as is typical in the literature, we consider only
 models where the dependence relation is acyclic.  It follows that
 given a \emph{context}, that is, a setting of the exogenous
 variables, the variables of all the endogenous variables are
 determined by the equations.

 \begin{example}
\label{ex:andlamp}
The model $M^{\wedge \mathrm{lamp}}$ has an exogenous variable $U$
whose range is a singleton $\s R(U)=\{u\}$
and three endogenous variables
$\s{V}=\{\textsc{Switch1},\textsc{Switch2},\textsc{Lamp}\}$,
all of whose ranges are 
$\{\textrm{on},\textrm{off}\}$.
In context $u$, 
$\textsc{Switch1}=\textsc{Switch2}=\textrm{on}$,
and the structural equation for $\textsc{Lamp}$ is
$$ 
\textsc{Lamp}
 = \begin{cases} \textrm{on} & \text{ if
    \textsc{Switch1}=on and \textsc{Switch2}=on}  \\ \textrm{off} &
  \text{ otherwise.} \end{cases} $$ 
\hfill \wbox
\end{example}

\def\LI{\mathrm{LI}}
In order to reason about causal models, HP define the following
language, which we will refer to as
$\LI$. 
We
start with atomic propositions of the form $X=x$, where $X$ is an
endogenous variable and $x \in \s{R}(X)$, and 
close off under conjunction and negation, and 
all   formulas of the form 
$\vec{[X} \gets \vec{x}]\phi$, where $X$ is a vector of endogenous
variables, which says that after intervening to set the variables in
$\vec{X}$ to  $\vec{x}$, $\phi$ holds.

To interpret the truth of a statement in a world described by a causal
model, we need to determine what values the variables can actually
take, given the structural equations and the values taken by exogenous
variables.
Given a context $\vec{u}$ and a setting $\vec{v}$ of the endogenous variables,
we say that the pair $(\vec u,\vec v)$ is 
\emph{compatible} with $M$ if the entry in  $\vec{v}$ for each
variable $V \in \s{V}$ 
is \emph{compatible} with $M$, that is, 
$\vec{v}(V) = F_V(\vec{u},\vec{v} - \{\vec{v}(V)\})$ for all $V
\in s(V)$. For models where the depends-on relation is acyclic,
as we assume here, there is a unique $\vec{v}$ for each $\vec{u}$
such that $(\vec u,\vec v)$ is compatible with $M$.

We can now give a semantics to propositional formulae
involving atomic propositions of the form $V=x$, where $V\in \s{V}\cup
\s{U}$, taken to mean that the variable $V$ takes value
$x$. Recursively, we set $(M,\vec u)\vDash X=x$ if the $X$-indexed
entry of $(\vec{u},\vec{v})$ has value $x$ 
when $\vec{v}$ is the unique assignment such that $(\vec u,\vec v)$ is
compatible with $M$.
We say $M \vDash X=x$ when $(M,\vec u)\vDash X=x$ for \emph{all} $\vec u$.
We can extend $\vDash$ to Boolean combinations of atomic formulas,
as well as quantifications over values (``$\forall v\in \s{R}(V)$'')
in the obvious way.

To give semantics to formulas of the form $[X\gets x]\phi$ in $M$, we
first need to define the model $M_{\vec X\leftarrow \vec x}$.
The model $M_{\vec{X} \gets \vec{x}}$ is just like $M$, except that the
structural equations for the variables in $\vec X$ are replaced by the
corresponding entries of $\vec x$;
that is, the equation for $X \in \vec{X}$ becomes $X=x$.
We then set $(M,\vec u)\vDash [\vec X\leftarrow \vec x] \varphi$ iff
$(M_{\vec X\leftarrow \vec x},\vec u)\vDash \varphi$. We read this
formula as ``if the variables $\vec{X}$ were set to the values $\vec{x}$,
then $\varphi$ would be true''. 

For example,
  the model $M^{\wedge \mathrm{lamp}}_{\textsc{Switch1}\leftarrow
    \textrm{off},\textsc{Lamp}\leftarrow \textrm{on}}$ is the same as
    $M$, 
except that in it, $F_{\textsc{Switch1}}=\textrm{off}$ and
$F_{\textsc{Lamp}}=\textrm{on}$. In particular, in this model,
whether the lamp is on does not depend on the state of either of the
switches.
We also have
$$M^{\wedge \mathrm{lamp}}\vDash
\textsc{Lamp}=\textrm{on} \wedge [\textsc{Switch1}\leftarrow
  \textrm{off}] \textsc{Lamp}=\textrm{off}:$$ the lamp is on, and if
the first switch were set to be switched off, the lamp would be off.

As is standard, we define $M\vDash \varphi$ to mean that
$(M,\vec u)\vDash \varphi$ for all contexts $\vec u$.
In particular, if there is only a single possible context $u$ (as is
the case in many of our examples),
then it is equivalent to $(M,u)\vDash \varphi$.
Moreover, since the ranges of all variables are finite,
we can take 
$\exists v\in \s{R}(X).\, \varphi$ to be syntactic sugar for the
disjunction $\bigvee_{x\in \s{R}(X)} \varphi[x/v]$, that is, the 
formula that is true iff it is true with some value from the range of $X$
substituted for $v$. When the range is clear from the context,
we may simply write $\exists v.\, \varphi$.

We can now state the HP definition of causality.  There are actually
three definitions (see \cite{HPearl01a,HP01b,Hal47,Hal48}).  We
consider the most recent one \cite{Hal47,Hal48}, called the
\emph{modified HP definition}, since it is simplest and seems most
robust. 

\begin{definition} \label{def:acm} $\vec{X}=\vec{x}$ is an
    \emph{actual cause} of $\varphi$ in $(M,\vec{u})$ if  
    \begin{description}
      \item[AC1.] $(M,\vec{u})\vDash \varphi$ and $(M,\vec{u})\vDash
  \vec{X}=\vec{x}$; 
\item[AC2.] There is a set $\vec{W}$ of variables in $\mathcal{V}$
  and a set of alternative values $\vec{x}'$ for the variables in
    $\vec{X}$ such that if $(M,\vec{u})\vDash \vec W = \vec w^*$, then
  $(M,\vec{u})\vDash [\vec{X}\leftarrow \vec{x}',\vec{W}\leftarrow
    \vec w^*]\neg \varphi$. 
  \item[AC3.] $\vec{X}$ is minimal; there is no strict subset $\vec{X}'$ of
 $\vec{X}$ for which AC2 holds. 
        \end{description}
    We write $\cause{\vec{X}=\vec{x}}{\varphi}$
to represent $\vec{X}=\vec{x}$ is a cause of $\phi$, so that
$(M,\vec{u}) \vDash \cause{\vec{X}=\vec{x}}{\varphi}$ if AC1--3 hold.
Note that conditions AC1--3 are all expressible in the language LI.
\hfill \wbox
\end{definition}

In Definition~\ref{def:acm}, HP assume that $\varphi$ is a Boolean
combination of atomic propositions.
In particular, $\varphi$ may not itself
be a formula of the form $\cause{\vec{X}=\vec{x}}{\varphi}$.
Since we want to reason about the expressive power of this
notion of causality, it is useful to explicitly define the language
obtained by augmenting propositional logic with it.
\def\LCone{\mathrm{LC}_1}
\begin{definition} \label{def:lcone}
A formula $\varphi$ is \emph{simple} if it is a Boolean combination
of atomic propositions of the form $X=x$.

A formula $\psi$ is a \emph{simple causal formula} if it is a Boolean
combination 
of atomic propositions of the form $X=x$ and causal statements
of the form $\cause{\vec{X}=\vec{x}}{\varphi}$, where $\varphi$ is simple.
The language of all simple causal formulae is called $\LCone$.
\hfill \wbox
\end{definition}

We give semantics to formulae in $\LCone$ by
converting them to formulas in $\LI$ and using  Definition~\ref{def:acm}.

\section{Nested causal statements}
\label{sec:nested}

Our goal is to investigate the role of nested causal statements such as
``$\vec A=\vec a$ is a cause of $\vec B=\vec b$ being a cause of $\varphi$''.
These statements have no formal counterpart in $\LCone$, but we can
define a language which includes them.
\newpage
\def\LCinf{\mathrm{LC}_\infty}
\begin{definition}
The language $\LCinf$ of \emph{nested causal formulae} is defined recursively
as follows:
\begin{itemize}
\item Simple formulae $\varphi$ are in $\LCinf$.
\item If $\varphi$ is in $\LCinf$, then so is $\cause{\vec{X}=\vec{x}}{\varphi}$.
\item Boolean combinations of formulae in $\LCinf$ are in $\LCinf$.
\end{itemize}
\hfill \wbox
\end{definition}
Since Definition~\ref{def:acm} does not depend on the structure of
$\varphi$ in $\cause{\vec X=\vec x}{\varphi}$, we can once again
use it to give a semantics to $\LCinf$ by evaluating the corresponding
formula in $\LI$.
As we now show, the inclusion of nested causal statements results in
$\LCinf$ being more expressive than $\LCone$
once we exclude a particular set of undesirable formulae.

\def\Mand{M^{\wedge\mathrm{lamp}}}
\def\Mnxor{M^{\neg\oplus\mathrm{lamp}}}

Let 
$M^{\neg\otimes\mathrm{lamp}}$ be the same as the model $M^{\wedge\mathrm{lamp}}$
from Example \ref{ex:andlamp}, except that 
$$ F_{\textsc{Lamp}} = \begin{cases} \textrm{on} & \text{ if
    \textsc{Switch1}=\textsc{Switch2}}  \\ \textrm{off} & \text{
    otherwise,} \end{cases} $$ 
so the lamp is on if both or neither switch is.
Intuitively, the two models $\Mand$ and $\Mnxor$ are quite distinct.
If we think of each model as representing different setups in which two people each control a light switch, then 
in $\Mand$, each participant has a veto on whether the light is on: if they choose to keep their switch ``off'',
then nothing the other person can do has any bearing on whether the light is on or not. On the other hand,
in $\Mnxor$, by flipping their own switch, each person can only temporarily toggle the light, perhaps to mess
with the other participant, but has no way of ensuring that the light will permanently remain in any particular state.
In security parlance, we could think of this as saying that in $\Mand$, each person has to independently
authorise the other to be able to influence the light. 

We want to show that natural simple causal formulae are
not sufficiently expressive to capture the difference between $\Mand$
and $\Mnxor$.
The
qualification \emph{natural}, however, does some work here: in order
to make this statement precise, we need to restrict the set of 
causal formulas that we consider.  
Specifically, 
the  intuition above  does not necessarily hold for some causal
statements where cause and effect refer to the same thing (such as
``it is raining because it is warm and raining'').
We typically do not make such statements; it seems strange to say
``$X=x$ is a cause of $X=x$'' or ``$X=x$ is a cause of $X=x$ and $Y=y$''.

\def\LConeNC{\LCone^\mathrm{nc}}
\def\LCinfNC{\LCinf^\mathrm{nc}}
\begin{definition} For a given formula $\phi$ 
let $\mathrm{Vars}(\phi)$ denote the set of variables (each $X$ in the atom $X=x$) in $\phi$.
    The formula $\causesimp\phi\psi$ is \emph{circular} if
$\mathrm{Vars}(\phi)\cap \mathrm{Vars}(\psi)\neq \varnothing$. 
    Let $\LConeNC$ and $\LCinfNC$ denote the \emph{n}on-\emph{c}ircular fragments
    of $\LCone$ and $\LCinf$ respectively.
\hfill \wbox
\end{definition}

As we now show, we cannot distinguish $\Mand$ and $\Mnxor$ using
non-nested non-circular formulae, but with nested non-circular
formulae, we can.

\begin{theorem} \label{thm:expressive} 
For all $\phi$ in $\LConeNC$, 
$$\Mand \vDash \phi \text{ iff } \Mnxor \vDash \phi,$$
but there exists
a nested non-circular causal statement $\psi\in \LCinfNC$ such that
$\Mand \vDash \psi \text{ and } \Mnxor \vDash \neg\psi.$
\end{theorem}
\begin{proof}
We start by showing that non-circular formulae cannot distinguish the models.
By exhaustive checking, we can confirm that $\Mand\vDash X=x$ iff $\Mnxor\vDash X=x$ for all $X\in \s{V}$ and $x\in  \s{R}(X)$.
It easily follows that $\Mand\vDash \varphi$ iff
$\Mnxor\vDash \varphi$ for simple formulae.

Formulae in $\LConeNC$ are Boolean combinations of atomic propositions
of the form 
$X=x$ and
non-circular causal formulae of the form $\cause{\vec X=\vec
  x}{\phi}$, where $\phi$ is a simple formula. 
Hence, if we can also establish that causal formulas are valid in
$\Mand$ iff they are 
valid in $\Mnxor$,
the result easily follows.
We do this by considering a number of cases.

If the variable $\textsc{Lamp}$ does not occur in
$\phi$, then the causal formula is false in both $\Mand$ and
$\Mnxor$.  To see this, note that
by non-circularity, $\phi$ can mention only
$\textsc{Switch1}$ 
and $\textsc{Switch2}$, but neither of these variables depends on any
other variable in either model. By non-circularity, $\vec{X}$ does not
mention the variables in $\phi$, so no change to the value of
$\vec{X}$ change the value of $\textsc{Switch1}$ or $\textsc{Switch2}$
(even if some variables $\vec{W}$ are fixed to their actual values).
Thus, AC2 cannot hold in either model, so the causal formula is false
in both models.

Now suppose that $\textsc{Lamp}$ occurs in $\phi$,
and hence $\textsc{Lamp}\not\in \vec X$. 
Consider the possible cases for $\vec{X}$.
\commentout{
\begin{itemize}
\item If $\vec{X}$ is empty, then the causal statement is false in every model, because $[]\varphi \Leftrightarrow \varphi$.
\item If $\vec{X}$ contains only one of $\textsc{Switch1}$ and $\textsc{Switch2}$, then the only nontrivial interventions (those not having $[\vec X\leftarrow \vec x]\varphi$ iff $\varphi$, and hence (AC1)$\Rightarrow\neg$(AC2)) that are minimal set the respective switch to ``off''. Since $\Mand_{\textsc{Switch1}\leftarrow \mathrm{off}} \vDash \textsc{Lamp}=\mathrm{off} \mathop{\reflectbox{\text{$\vDash$}}} \Mnxor_{\textsc{Switch1}\leftarrow \mathrm{off}}$ and likewise for $\textsc{Switch2}$, by induction on Boolean combinations, the same simple formulae $\varphi$ are valid in both intervened-upon models,
and hence the same statements of the form $[\textsc{Switch}i\leftarrow \mathrm{off}] \varphi$ are.
\item If it contains both, then minimality (AC3) fails, so the causal statement is false in both: the only atoms $\phi$ can include are those imputing some value to \textsc{Lamp}, and in both $\Mand$ and $\Mnxor$, that value can only change in one way which is already achieved by just changing one of the two switches.
\end{itemize}
\end{itemize}
}
If $\vec{X} = \vec{x}$ contains either $\textsc{Switch}1 =
\mathrm{off}$ or $\textsc{Switch}2 = \mathrm{off}$, then $\vec{X} =
\vec{x}$ is false in both models, so AC1 fails, and the causal formula
is false in both models.  If $\vec{X} = \emptyset$, then AC2 must fail
in both models (since no change in the value of a variable in
$\vec{X}$ can cause a change in the truth value of $\phi$.  If
$\vec{X} = \vec{x}$ is  $\textsc{Switch}1 = \mathrm{on}$, then changing 
$\textsc{Switch}1$ to  $\mathrm{off}$ (while possibly keeping some
variables fixed at their actual values) has the same effect on the
truth value of all the variables in both models, so AC2 will either
hold in both models or in neither, and AC3 trivially holds in both.  A
similar argument works if 
$\vec{X} = \vec{x}$ is  $\textsc{Switch}1 = \mathrm{on}$.  Finally, if 
$\vec{X} = \vec{x}$ is  $\textsc{Switch}1 = \mathrm{on} \land
\textsc{Switch}1 = \mathrm{on}$, then the causal formula must be false
in both models.
Either at least one of AC1 and AC2 is violated,
or we must have $\varphi \Leftrightarrow \textsc{Lamp}=\mathrm{on}$ in
both models
as AC1 necessitates $\Mand,\Mnxor\vDash \varphi$,
AC2 necessitates $\Mand,\Mnxor\vDash [\textsc{Switch}1 \leftarrow v',
  \textsc{Switch}2\leftarrow w'] \neg \varphi$, 
and by non-circularity, $\varphi$ can mention only $\textsc{Lamp}$.
But in both models, $[\textsc{Switch}i\leftarrow \mathrm{off}] \neg (\textsc{Lamp}=\mathrm{on})$
for $i\in \{1,2\}$, so $\textsc{Switch}i=\mathrm{on}$ is already a
cause of $\textsc{Lamp}=\mathrm{on}$ and AC3 (minimality) is violated.

We now show that the models can be distinguished by nested non-circular
formulae. We claim that
\begin{align*}
\Mand \vDash & \,\, \textsc{Switch1}=\mathrm{on} \\ 
& \rightsquigarrow ( \exists v.\, \textsc{Switch2}=\mathrm{on}
\rightsquigarrow \textsc{Lamp}=v ), 
\end{align*}
but
\begin{align*}
\Mnxor \not\vDash & \,\, \textsc{Switch1}=\mathrm{on} \\ 
 & \rightsquigarrow ( \exists v.\, \textsc{Switch2}=\mathrm{on} \rightsquigarrow \textsc{Lamp}=v ).
\end{align*}
In both models, both  $\textsc{Switch1}=\mathrm{on}$ and
 $\textsc{Switch2}=\mathrm{on} \rightsquigarrow
\textsc{Lamp}=\mathrm{on}$ are valid,
as intervening to set $\textsc{Switch2}\leftarrow \mathrm{off}$ results in $\textsc{Lamp}=\mathrm{off}$.
However, if we intervene to set $\textsc{Switch1}\leftarrow \mathrm{off}$, we have
$\Mand \not\vDash \textsc{Switch2}=\mathrm{on} \rightsquigarrow \textsc{Lamp}=v$ for any $v$,
as $\textsc{Lamp}=\mathrm{off}$ regardless of the setting of $\textsc{Switch2}$.
This is not the case in $\Mnxor$, as there we have 
$$\Mnxor\vDash [\textsc{Switch1}\leftarrow\mathrm{off}](\textsc{Switch2}=\mathrm{on} \rightsquigarrow
\textsc{Lamp}=\mathrm{off}).$$ 
Intervening further to set $\textsc{Switch2}\leftarrow\mathrm{off}$ results in the lamp turning
back on, as both switches are in the same position again.
\end{proof}

\begin{remark}
  It is worth noting that non-circularity is really a necessary condition here.
  Let $\phi$ be the formula $\textsc{Switch1}=\mathrm{on}\wedge
\textsc{Switch2}=\mathrm{on}\wedge \textsc{Lamp}=\mathrm{on}$; let
$\psi$
be the formula $\textsc{Switch1}=\mathrm{on}\vee
 \textsc{Switch2}=\mathrm{on}\vee \textsc{Lamp}=\mathrm{on}$.
 Note that the formula $\phi \rightsquigarrow \psi$ is circular.
 Moreover, $\Mnxor \vDash \phi \rightsquigarrow \psi$, 
as we can intervene to set all the variables to \textit{off}, but
 $\Mand \not\vDash \phi\rightsquigarrow \psi$,
  because AC3 is violated: we have 
$\Mand \vDash
 \textsc{Switch1}=\mathrm{on}\wedge \textsc{Switch2}=\mathrm{on}\ 
 \rightsquigarrow \psi$.
The corresponding
 causal statement does not hold in $\Mnxor$ because
$\textsc{Lamp}=\mathrm{on}$ after setting both switches
\textit{off}. 

In fact, as we show in the appendix, similar
circular formulae in $\LCone$ can
distinguish any two distinct causal models.
\hfill \wbox
\end{remark}

While Theorem \ref{thm:expressive} formalizes our claim that
$\Mand$ and $\Mnxor$
can be distinguished by nested causal formulae,
the existential quantification in the distinguishing
statement makes it somewhat difficult to understand exactly what it is
about the models that is different.
The formula seems to be saying that the first switch being on
is a cause of the second switch being on causing \emph{something}
about the lamp.

This \emph{something} can not be expressed as \textsc{Lamp} taking
a particular value: if we took it to be $\textsc{Lamp}=\mathrm{on}$, then
the resulting nested causal statement would be valid in both
$\Mand$ and $\Mnxor$ by AC2, as after intervening to set $\textsc{Switch1}\leftarrow \mathrm{off}$,
$\textsc{Lamp}$ is not $=\mathrm{on}$ anymore, and so the effect
$\textsc{Switch2}=\mathrm{on}\rightsquigarrow \textsc{Lamp}=\mathrm{on}$ is
becomes false by AC1.
On the other hand, if we took it to be $\textsc{Lamp}=\mathrm{off}$, then
the nested causal statement would be invalid in both,
as AC1 requires that the effect $\textsc{Switch2}=\mathrm{on}\rightsquigarrow \textsc{Lamp}=\mathrm{off}$
is valid in the real world, and another application of AC1
necessitates $\textsc{Lamp}=\mathrm{off}$, but the lamp is really on.

In the next section, we argue that
the deliberately fuzzy wording (``\ldots causing something
about the lamp'') actually captures the existential quantification,
which sidesteps the issue of not being able to choose a fixed value,
and that in the presence of nested causality, it turns out to
sometimes make sense
to avoid committing to particular values.

\section{Variables, rather than facts, as causes}
\label{sec:variableascauses}

We previously observed that out of all
the possible causal statements of the form $\cause{A=a}{(C=c)}$, where
$a\in
\mathcal{R}(A)$, $c\in \mathcal{R}(C)$, only one is potentially true,
namely the one where $a$ and $c$ are the actual values that $A$ and
$C$, respectively, take in the context. There is therefore a sense in
which specifying $a$ and $c$ is redundant: we could unambiguously
interpret a formula like $\causesimp{A}{C}$ as meaning
$\cause{A=a}{(C=c)}$ in each context $u$ with $a$, $c$ such that
$(M,u)\vDash A=a \wedge C=c$.  
Things are not so simple in the case of nested causality.
Suppose that, in the causal formula, 
$\cause{\vec C=\vec c}{\varphi}$,
$\phi$ itself is a causal formula of the form $\cause{\vec A=\vec
    a}{(\vec B=\vec b)}$.
\commentout{
It follows from  AC1 that this formula is false if either 
$\vec A=\vec a$ or $\vec B=\vec b$ is false.
As in the proof of Theorem \ref{thm:expressive},
depending on the setting of $\vec C$, this could fail because $\vec A$
or $\vec B$ depends on $\vec C$. We would then find that
 $\cause{\vec C=\vec c}{(\cause{\vec A=\vec a}{\vec B=\vec b})}$ is actually
false for a reason such as ``if $\vec C$ were instead set to $\vec c'$, $\vec A=\vec a$
(resp. $\vec B=\vec b$) would not be true to begin with''.
But this may not be an appropriate answer to the
question we are actually interested in 
when considering a nested causal statement. The proof of the
Theorem \ref{thm:expressive} is an example where we really want to query
whether the outer cause (corresponding to $\vec C=\vec c$ here)
has an impact on the causal relationship between $\vec A=\vec a$ and
the value of the variable $\vec B$, whatever this value may be.
This indifference towards the value can equally affect the inner cause,
as illustrated by the following formal account of an example from the
introduction.}
Now the values of $a$ and $b$ for which the formula is true depend on
the value $c'$ to which $C$ is set when evaluating the
counterfactual.  As the following example shows, this can play a
critical role.

\begin{example} \label{ex:farmer}
    Let $M$ be the following model of the farmer story from the introduction.
A farmer may relocate to
  Texas ($R=1$) or stay in New York ($R=0$), and
this will impact the level of drought his crops are exposed to ($W=0$ standing
for drought, $W=1$ standing for normality, and $W=2$ standing for flooding).
 If he
  relocates, his crops will suffer dry weather; otherwise, they
  will be flooded: 
$$W=\begin{cases} 2& \text{ if $R=0$} \\ 0&
    \text{otherwise.} \end{cases}$$
 The crops, however, can survive ($C=1$) only
  if the weather is fair ($W=1$): 
  $$C=\begin{cases} 1 & \text{ if $W=1$} \\ 0 & \text{ otherwise.} \end{cases}$$
  Intuitively, assuming the farmer relocated, we hold that the statement ``the weather caused the crops to fail'' is true; but the statement ``because the farmer relocated to Texas, the weather caused the crops to fail'' is false.
Finally, we add a single exogenous variable $U$ with singleton range.
In this context, $R=1$.
  \hfill \wbox
\end{example}

What should the formal interpretation of the nested statement be?
If we take the interpretation of ``the weather caused the crops to fail''
to be $\cause{W=0}{(C=0)}$, then this formula is indeed true (as
$M\vDash [W\leftarrow 1]C=1$). 
If we then interpret the nested statement as
$$\cause{R=1}{(\cause{W=0}{(C=0)})},$$
then this is in fact true as well: $M\vDash [R\leftarrow 0] (W=2)$,
and hence $M\vDash [R\leftarrow 0]\neg (\cause{W=0}{(C=0)}$!
It does not help to interpret the causal formula as
$$\cause{R=1}{(\cause{W=2}{(C=0)})}.$$
This statement is false for vacuous reasons: $R=1$ is not a cause
because $\cause{W=2}{C=0}$ is false, as $W$ is not 2.
What seems to capture this example best is to use of existential
quantification: 
$$ \cause{R=1}{(\exists w.\, \cause{W=w}{(C=0)})}. $$

The common feature of this example and the distinguishing formula
in the proof of Theorem~\ref{thm:expressive} is that we do not
know in advance what value is appropriate to impute to variables
in the inner causal statement: when considering the outer statement,
we need to consider all possible counterfactual scenarios concerning
its cause, but the values of the variables mentioned in the inner
statement may be different in each of those scenarios. 
Therefore, we make the following definition, which allows us to capture
the notion that $A$ taking the situationally appropriate value, whatever it is in
the scenario that we might be considering for a particular causal (sub)formula,
is the cause of the $B$ taking whatever value it happens to take. In the next
section, we will see examples of several security properties that
are expressed more succinctly with this notation, and in some cases
cannot be expressed adequately without the
 existential quantification at all.
\begin{definition}
Suppose $\vec A$ and $\vec B$ are lists of variables of an appropriate causal model.
Let $\causesimp{\vec A}{\vec B}$ denote the formula
$$\exists \vec a\in \s{R}(\vec A), \vec b\in \s{R}(\vec B).\,
\cause{\vec A=\vec a}{(\vec B=\vec b)}.$$ 
We can define
$\cause{\vec A=\vec a}{\vec B}$ and $\causesimp{\vec A}{(\vec B=\vec b)}$
analogously. 
\hfill \wbox
\end{definition}

Using this definition, we can now state the
causal statement of Example~\ref{ex:farmer} in a way that
mirrors the natural-language version, by saying
$\causesimp {R=1} {(\causesimp W {(C=0)})}$,
or even more succinctly as $\causesimp R {(\causesimp W C)}$.
Likewise, the distinguishing formula from the proof
of Theorem~\ref{thm:expressive} can now be stated
as $\causesimp {\textsc{Switch1}=\mathrm{on}} 
{(\causesimp {\textsc{Switch2}=\mathrm{on}} {\textsc{Lamp}})}$,
or more compactly as $\causesimp {\textsc{Switch1}} 
{(\causesimp {\textsc{Switch2}} {\textsc{Lamp}})}$.

\section{Examples of causal security}

Now that we have given the definitions, we are ready
to revisit several examples of propositions about security that we want to
argue are naturally viewed as nested causal statements.

\begin{example}
A government employee has the authority to declassify government secrets
and release them to the press. The employee turns out to be corrupt: if
someone pays him a sufficient amount of money, he will declassify a secret
and have it published in the press. As it happens, a UFO enthusiast
community scrapes together a bribe and pays the employee, who subsequently
publishes the announcement that the government has been in contact
with aliens.
\hfill \wbox
\end{example}

In what sense can we say that something inappropriate occurred?
By assertion, we considered it permissible for the employee to declassify
and release the secret (and thus for the truth about aliens to be a
cause of the press release). In a free market economy, tax regulations 
not withstanding, people are free to give money
to whomever they please. Lastly, had the UFO enthusiasts instead paid
a struggling newspaper directly to announce that the government found UFOs,
this would also not be problematic.
The problem here is that that whether the information was released
depended on whether the bribe was paid. In other words, the problem is
that the bribe was a cause of the government being in contact with aliens
being a cause of the press release.

Formally, we can represent the
example as a causal model $M^{\textrm{aliens}}$ 
with a single exogenous variable $U$ whose range is a singleton,
and three endogenous
binary variables
$S$, $P$, and $B$, representing whether the government is secretly in
contact with 
aliens, whether there is a press article to the effect, and whether the UFO
enthusiasts paid a bribe respectively. Due to the employee's corruption,
we have $$F_P=\begin{cases} S & \text{ if $B$=1} \\ 0 & \text{
otherwise.}\end{cases}$$ 
 The undesirable causal relationship
 then is represented by the formula $\causesimp {B} {(\causesimp S P)}$.

The security property being violated here is an instance of 
Zdancewic and Myers's notion of \emph{robust declassification} \cite{zdancewic01,MSZ04,nmifc}.
Roughly speaking, a declassification (release of a secret) is considered
\emph{robust} if whether the declassification occurred was not up to
an untrusted
actor. 
What is considered a secret and what actors are trusted (to
declassify the 
secret) has to be
specified as part of the security policy.
In causal terms, we can say that a system satisfies robust declassification
if there is no instance of an \emph{untrusted} variable (such as the UFO
enthusiasts' decision to pay) being a cause of a \emph{secret} variable
being a cause of a \emph{public} variable.
Which variables belong to each of the three classes has to be
specified as part of the security policy. Intuitively, secret variables
are those for which we would consider it \emph{a priori} unacceptable
for parties 
unaffiliated with the principal that the security policy seeks to protect
 to learn their value, unless this was
explicitly desired by the system designer. We can assume that their value
is not directly visible to outsiders, for otherwise the system would be trivially insecure.
Public variables are all those that are assumed to be visible to outside
observers. Trusted variables are those whose value is taken to be under 
the control of the principal; untrusted variables  may have had
their value influenced by outsiders whose interests may not
align with those of the principal.

The presence of an untrusted 
variable as a 
cause can turn an otherwise acceptable
causal relationship unacceptable. Conversely,
the presence of a trusted cause can turn an otherwise unacceptable causal
relationship acceptable.

\begin{example} 
Alice's computer-illiterate boss, Bob, has asked Alice to fix his
computer.
While she is at it, she realises that his desktop background
is the default colour (say, white). She decides to set the desktop
background to her favourite colour.  (For simplicity, in the remainder
of the discussion, we asume that there are only two possible colours.)
Consider two scenarios:
\begin{enumerate}
\item Alice is sensitive to the circumstance that she is working
  on somebody else's machine. Her understanding with Bob is that has
  she is entitled to change some setting (like background colour)
  unless Bob explitly tells her not to.
\item Alice is quite fed up with Bob's lack of taste and clueless
  management.
  If Bob were to tell her to leave the desktop background
alone, she would just get spiteful and instead set it to the opposite of her
favourite colour. \hfill \wbox
\end{enumerate}

\end{example}
Our intuition says that in the first case, the colour change was
(implicitly) authorised
by Bob. Were he to complain about it, Alice could rightly respond that
she wouldn't have done it if he had told her not to, and he had ample
opportunity to. 
On the other hand, Bob would not be wrong to complain about her
meddling and insubordination in the second case. This is not just a matter
of control; knowing Alice's behaviour, Bob can make her set the background
to any colour he prefers by tactically choosing whether to tell her to back off.

Formally, we could represent the cases as causal models $M^\mathrm{obedient}$
and $M^\mathrm{defiant}$ with
a single exogenous variable whose range is a singleton, and
 three
binary endogenous variables, $A$ representing Alice's favourite colour,
$B$ representing whether Bob tells Alice 
that it is okay to change the colour,
and $C$ representing the resulting background colour. In the first,
``obedient'' case, we have 
$$F_C^\mathrm{obedient}=\begin{cases} A & \text{ if $B$=1} \\ 1 &
\text{ otherwise.}\end{cases}$$ 
On the other hand, in the ``defiant'' case,
$$F_C^\mathrm{defiant}=\begin{cases} A & \text{ if $B$=1} \\ 1-A &
\text{ otherwise.}\end{cases}$$ 
($F_A$ and $F_B$ just set $A$ and $B$ to Alice and Bob's actual
actions in each case.) 

It is easy to check that these two models are just relabellings of the 
models $\Mand$ and $\Mnxor$ from earlier, respectively; thus,
we have $M^\mathrm{obedient} \vDash \causesimp B {(\causesimp A C)}$, but
$M^\mathrm{defiant} \not\vDash  \causesimp B {(\causesimp A C)}$. More generally,
we could consider this an instance of a security policy that we could call
\emph{authorisation}: the untrusted variable
$A$ 
is only a cause of the privileged outcome $C$ if this causal
relationship itself had
a trusted cause $B$, interpreted as the causation happening at $B$'s
pleasure, 
with $B$ having the option to prevent it and choosing to not making use of it.

The ``authorisation'' construction that we have just described can be easily iterated
to generate more complex meaningful examples.

\begin{example} Suppose Bob is not present during Alice's fixing
of his computer, and instead has told Alice to
let Bob's secretary Dylan supervise her.
Would Alice listen to Dylan if he were to tell her to leave Bob's desktop
background unchanged (which, in fact, he doesn't)?
Once again, consider two cases:
\begin{enumerate}
\item Alice respects Bob's delegation of authority, and sets the desktop
to her preferred background colour only if Dylan doesn't tell her
to leave it alone. If Bob had instead told her not to listen to Dylan, 
she would have strictly acted according to her own best judgement,
and set the desktop background to her preferred colour no matter what he said.
\item Alice thinks much more highly of Dylan than the boss they
work for, and
will listen to him even if Bob tells her not to. As it happens, Bob
trusts Alice's artistic judgement much more than Dylan's,
and will be quite displeased to hear that his overbearing underling
stopped Alice from setting him up with an artfully chosen
background.
\end{enumerate} 
\hfill \wbox
\end{example}
In both scenarios, the final setting of the desktop background is caused
by Alice's preferred colour, and this causal relationship in turn
is caused by Dylan's acquiescence. What intuitively 
distinguishes the two scenarios (and would continue distinguishing them
if Alice's defiance were to come to the fore under some combinations of
Bob's and Dylan's instructions) is whether Dylan's control over
this itself was ``at Bob's pleasure'', that is, could have been
vetoed by Bob, or Bob's authority was usurped.
Formally, we can capture them as two causal models,
both with a single exogenous variable whose range is a singleton,
four endogenous variables, and 
\begin{eqnarray*}
F_C^1&=&\begin{cases} A & \text{ if $B=1$ and $D=1$} \\ 1 & \text{ else}\end{cases} \\
F_C^2&=&\begin{cases} A & \text{ if $D=1$} \\ 1 & \text{ else}\end{cases}. 
\end{eqnarray*}
We then find while $M^1 \vDash \causesimp B {(\causesimp D
    {(\causesimp A C))}}$, 
we have $M^2 \not\vDash \causesimp B {(\causesimp D {(\causesimp A C))}}$.

This construction can be iterated further in a straightforward manner, allowing
us to express any number of steps of delegation. Such chains of delegation are
often considered in \emph{authorisation logics} (see e.g. Abadi's survey \cite{abadiacl}),
but rarely given a formal semantics \cite{andrewauthl}, let alone one
that can be applied in 
a setting as general as causal models.

\section{Conclusions}

We have shown that a variety of security properties can be expressed as
nested causal statements. To give a formal account of such statements,
we extended the Halpern-Pearl framework of causality to allow formulae
that may themselves refer to causal relationships as effects of causal
statements. As we have shown, the language of nested causal statements
thus obtained is more expressive than the language of simple causal statements
that the HP framework normally deals in. This extension also led us to
revisit a particular design assumption of the HP framework, namely
that it is always appropriate to have causal statements refer to variables
taking particular values as causes and effects. We have argued that in nested 
causal statements, it is often more natural to implicitly existentially quantify
over values, using just the variables as causes and effects (interpreted
as meaning that whatever value the variable actually takes is the cause or effect).

We view this paper as laying the groundwork for the characterization
of security properties using (nested) causality.  
In future work, we plan to formalize various security properties using
causality,
show how to represent arbitrary programs as causal models,
and use these causal models to provide efficient sound procedures for verifying
that the programs satisfy these security properties.

\appendix

\section*{Appendix: The power of circular causal statements}
\setcounter{section}{1}
\label{app:power}

\begin{proposition} \label{prop:circlepower}
Let $(M^1,\vec u_1)$ and $(M^2,\vec u_2)$ be two contexts of recursive models with the same signature $( \s U,\s V,\s R)$, same immediate dependency ordering, same unique consistent assignment $\vec v$ and different structural equations.
Then there is a simple causal formula $\cause{\vec X=\vec x}\phi \in \LCone$ that is valid in $(M^1,\vec u_1)$ but not in $(M^2,\vec u_2)$.
\end{proposition}
\begin{proof}
    Let $D$ be  a minimal variable in the dependency ordering
    (which is th same for both models) 
whose structural equation takes different values in $(M^1,\vec u_1)$
and $(M^2,\vec u_2)$ on some input (assignment to endogenous variables
that the structural equation depends on) $\vec v'$, say $F^1_D(\vec
u_1, \vec v')\neq F^2_D(\vec u_2,\vec v')$, and take $\vec X\subseteq
\s V$ to be the strict descendants of $D$ in the dependency ordering.
Let $\vec x$ be the corresponding entries of $\vec v$,
and let
$$\phi=\neg (D=F^1_D(\vec u_1,\vec v')) \vee \bigvee_{Y\in \vec
X}  \neg (Y=\vec v'(Y)).$$
That is, $\phi$ says that $D$ or one of its descendants has a
different value than it would have on input  $\vec v'$. 

Then $(M^1,\vec u_1) \vDash \vec X = \vec x$ and $(M^2,\vec u_2) \vDash \vec X=\vec x$ by construction, and $(M^1,\vec u_1) \vDash \phi$, $(M^2,\vec u_2) \vDash \phi$ 
because $\vec v\neq \vec v'$ (recall we are assuming that the two models have the same unique consistent assignment, so $F^1_D(\vec v)=F^2_D(\vec v)=\vec v(D)$) and so at least one of the disjuncts must be true in both models (which are only consistent with the assignment $\vec v$).
Moreover, $(M^1,\vec u_1) \vDash [\vec X \leftarrow \vec x'] \neg \phi$, where $\vec x'$ are the entries of $\vec v'$ corresponding to $\vec X$: the clauses of the disjunction are satisfied by the intervention directly setting $Y=\vec v'(Y)$ for all $Y\in \vec X$, and $D=F^1_D(\vec v')$ as a consequence of the structural equations of $M^1$. So either $(M^1,\vec u_1) \vDash \cause{\vec X=\vec x}\phi$, or this only fails AC3 and so there is some proper subvector $\vec X'$ such that $(M^1,\vec u_1) \vDash \cause{\vec X'=\vec x}\phi$ for which this is true.
However, for all $\vec x''$, $(M^2,\vec u_2) \vDash [\vec X \leftarrow \vec x''] \phi$: either $\vec x'' \neq \vec x'$, in which case one of the disjuncts about $Y\in \vec X$ is true in $(M^2_{\vec X\leftarrow \vec x''},\vec u_2)$, or $\vec x''=\vec x'$, in which case we have $(M^2_{\vec X\leftarrow \vec x''},\vec u_2)\vDash D=F^2_D(\vec v')$, and hence $\neg (D=F^1_D(\vec v'))$. Since this is true for all values we could assign to $\vec X$, restricting to a subvector $\vec X'$ of $\vec X$ will not help, and $(M^2\vec u_2)\not\vDash \cause{\vec X'=\vec x}\phi$ for all $\vec X'\subseteq \vec X$.
\end{proof}

Note that if the two models do not have the same unique consistent
assignment $\vec v$, we don't even need causal statements to
distinguish them: just use a single atomic proposition  about a
variable on which 
their consistent assignments differ.

\bibliographystyle{named}
\bibliography{secondorder}

\end{document}